\documentclass[conference]{IEEEtran}
\IEEEoverridecommandlockouts
\usepackage{cite}
\usepackage{amsmath,amssymb,amsfonts}
\usepackage{graphicx}
\usepackage{textcomp}
\usepackage{xcolor}
\def\BibTeX{{\rm B\kern-.05em{\sc i\kern-.025em b}\kern-.08em
    T\kern-.1667em\lower.7ex\hbox{E}\kern-.125emX}}
\usepackage{subcaption}
\usepackage{url}
\usepackage{multirow}
\usepackage{enumitem}
\usepackage{algpseudocode}
\usepackage{algorithm}

\usepackage{amsthm}
\newtheorem{prop}{Proposition}
\usepackage{multicol}
\usepackage{ulem}
\usepackage{soul}
\usepackage{balance}
\usepackage{verbatim}
\usepackage{todonotes}
\usepackage{float}
\usepackage{hyperref}
\usepackage{rotating}
\begin{document}

\title{ZK-GanDef: A GAN based Zero Knowledge Adversarial Training Defense for Neural Networks\\
\thanks{This work is submitted to DSN 2019 as a regular paper. \textcolor{black}{The official implementation is open-source on Github. \url{https://github.com/GuanxiongLiu/DSN-ZK-GanDef.git}}}
}

\author{\IEEEauthorblockN{Guanxiong Liu}
\IEEEauthorblockA{\textit{ECE Department} \\
\textit{New Jersey Institute of Technology}\\
Newark, USA \\
gl236@njit.edu}
\and
\IEEEauthorblockN{Issa Khalil}
\IEEEauthorblockA{\textit{QCRI} \\
\textit{Hamad bin Khalifa University}\\
Doha, Qatar \\
ikhalil@hbku.edu.qa}
\and
\IEEEauthorblockN{Abdallah Khreishah}
\IEEEauthorblockA{\textit{ECE Department} \\
\textit{New Jersey Institute of Technology}\\
Newark, USA \\
abdallah@njit.edu}
}

\newcommand{\issa}[1]{\todo[inline,color=brown!40]{Issa: #1}}
\newcommand{\guan}[1]{\todo[inline,color=green!40]{Guanxiong: #1}}

\maketitle

\begin{abstract}
Neural Network classifiers have been used successfully in a wide range of applications. However, their underlying assumption of attack free environment has been defied by adversarial examples. Researchers tried to develop defenses; however, existing approaches are still far from providing effective solutions to this evolving problem. In this paper, we design a generative adversarial net (GAN) based zero knowledge adversarial training defense, dubbed \textbf{ZK-GanDef}, which does not consume adversarial examples during training. Therefore, ZK-GanDef is not only efficient in training but also adaptive to new adversarial examples. This advantage comes at the cost of small degradation in test accuracy compared to full knowledge approaches. Our experiments show that ZK-GanDef enhances test accuracy on adversarial examples by up-to 49.17\% compared to zero knowledge approaches. More importantly, its test accuracy is close to that of the state-of-the-art full knowledge approaches (maximum degradation of 8.46\%), while taking much less training time.
\end{abstract}


\section{Introduction}\label{sec:introduction}

Due to the surprisingly good representation power of complex distributions, \textcolor{black}{neural network (NN)} classifiers are widely used in many tasks which include natural language processing, computer vision and cyber security. For example, in cyber security, NN classifiers are used for spam filtering, phishing detection as well as face recognition \cite{rowley1998neural} \cite{abu2007comparison}. However, the training and usage of NN classifiers are based on an underlying assumption that the environment is attack free. Therefore, such classifiers fail when adversarial examples are presented to them. 

Adversarial examples were first introduced in 2013 by Szegedy et. al \cite{szegedy2013intriguing} in the context of image classification. They show that adding specially designed perturbations to original images could effectively mislead  fully trained NN classifier. For example, in Figure~\ref{fig:fgm-example}, the adversarial perturbations added to the image of panda are visually insignificant to human eyes, but are strong enough to mislead the classifier to classify it as gibbon. Yet, more scary, the research shows that adversary could arbitrarily control the output class through carefully designed perturbations and can achieve high success rate against Vanilla classifiers, i.e., classifiers without defenses \cite{carlini2016towards}\cite{kurakin2016adversarial}\cite{madry2017towards}.

\begin{figure}[tb]
\centering
\begin{minipage}[c]{.4\textwidth}
\centering
    \includegraphics[width=\linewidth]{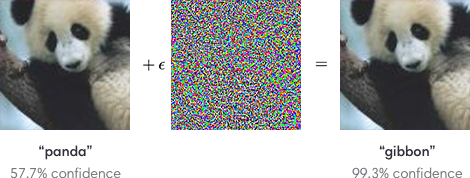}
    \caption{Fast Gradient Sign Example \cite{goodfellow2014explaining}}
    \label{fig:fgm-example}
\end{minipage}
\vspace{-5mm}
\end{figure}

Current defensive mechanisms against adversarial examples can be categorized into three different groups \cite{meng2017magnet}\cite{papernot2016distillation}. The approaches of the first group utilize augmentation and regularization to enhance test accuracy on adversarial examples. The idea here is to improve the generalization of the model as a defense against adversarial examples \cite{papernot2016distillation}. Approaches in the second group try to build protective shells around the classifier to either identify adversarial examples and filter them out, or reform perturbations and rollback to original images \cite{meng2017magnet}\cite{samangouei2018defense}. The approaches in the last group retrain NN classifiers with adversarial examples to recognize and correctly classify perturbed inputs \cite{kurakin2016adversarial}. The intuition here is that by observing some adversarial examples with their ground truth, the NN classifiers learn the patterns of adversarial perturbations and adapt to recognize similar ones.

Current defenses enhance test accuracy of existing NN models on adversarial examples and help us better understand the problem. However, such defenses are still far from wining the battle against this continuously evolving problem. Among different defenses, adversarial training with iterative adversarial examples is shown to be the state-of-the-art choice \cite{athalye2018obfuscated}. However, such defense requires too much computation to generate iterative adversarial examples during training. Based on \cite{kannan2018adversarial}, the adversarial training with iterative adversarial examples requires a cluster of GPU servers on Imagenet dataset. \textcolor{black}{Although there are defense methods that do not rely on adversarial training \cite{rajabi2018towards}\cite{xu2017feature}, these methods only address weaker attack scenarios (black-box and gray-box).}

\textcolor{black}{With these limitations on existing defenses, researchers start a new line of research on \textbf{zero knowledge} adversarial training which is independent of adversarial examples.} The idea here is to replace adversarial examples with random noise perturbations while retraining of NN classifiers \cite{kannan2018adversarial}. The intuition is to trade small decrease in accuracy for better scalability, efficiency and quick adaptability. In this work, the adversarial training approaches which utilize adversarial examples are denoted as \textbf{full knowledge} adversarial training, in contrast with zero knowledge ones.

As we show in our evaluation, existing zero knowledge adversarial training approaches, clean logit pairing (CLP) and clean logit squeezing (CLS) \cite{kannan2018adversarial}, suffer from poor prediction accuracy. In this work, we propose a GAN based zero knowledge adversarial training defense, dubbed \textbf{ZK-GanDef}. ZK-GanDef is designed based on adversarial training approach combined with feature learning \cite{louppe2017learning}\cite{xie2017controllable}\cite{lample2017fader}. It forms a competition game of two NNs: a classifier and a discriminator. We show analytically that the solution of this competition game generates a classifier which usually makes right predictions and only relies on perturbation invariant features. We conduct extensive set of experiments to evaluate the performance and the prediction accuracy of ZK-GanDef on MNIST, Fashion-MNIST and CIFAR10 datasets. Compared to CLP and CLS, the results show that ZK-GanDef has the highest test accuracy in classifying different \textbf{white-box adversarial examples} with significant superiority.

Our contributions can be summarized as follows: 

\begin{itemize}
    \item We design a GAN based zero knowledge adversarial training defense, ZK-GanDef, which utilizes feature selection to enhance test accuracy on adversarial examples.
    
    \item \textcolor{black}{We provide a mathematical intuition for the competition game used in ZK-GanDef that its trained classifier usually makes right predictions based on perturbation invariant features.}
    
    \item We empirically show that ZK-GanDef significantly enhances the test accuracy on adversarial examples over state-of-the-art zero knowledge adversarial training defenses.
    
    \item Existing work only tests CLP and CLS on small datasets like MNIST. In this work, we empirically show that CLP and CLS do not scale well to complex datasets such as CIFAR10. In contrast, we show that ZK-GanDef can defend adversarial examples in such complex datasets.
    
    \item We empirically show that ZK-GanDef can achieve comparable test accuracy to the state-of-the-art full knowledge defenses. At the same time, it significantly reduces the training time compared to full knowledge defenses. For example, its training time is 92.11\% less than that of PGD-Adv on MNIST dataset.
\end{itemize}

The rest of the work is organized as follows. Section \ref{sec:background} presents background material. The design and mathematical proof of ZK-GanDef are given in Section \ref{sec:defense}. Section \ref{sec:experiment} presents the test-bed design and the experimental settings. The evaluation results are shown in Section \ref{sec:results}. Section \ref{sec:conclusion} concludes the paper and Section \ref{sec:future} presents our future direction. 

\section{Background and Related Work} \label{sec:background}

In this section, we introduce background material about adversarial example generators and defensive mechanisms for better understanding of the concepts presented in this work. We also provide relevant references for further information about each topic.

\subsection{Generating Adversarial Examples}

The methods for generating adversarial examples against NN classifiers can be categorized according to several different aspects. In one aspect, these methods could be distinguished by the adversary's knowledge of the target NN classifier. In White-box methods, the adversary is assumed to have full knowledge about the target NN classifier (structure, parameters and inner status), and hence the generated examples are called white-box adversarial examples. On the other hand,  black-box methods assume that the adversary has no access to the inner information of the target NN classifier, and hence, the generated examples are called black-box adversarial examples. On another aspect, adversarial example generation methods could be categorized into single-step or iterative methods according to the process of generating the examples. Single-step methods only run gradient descent (ascent) algorithm once when solving the proposed optimization problem, while iterative methods repeat the computation several times until hitting predefined convergence thresholds.

Based on previous works, an adversarial example generator is generally formulated as an optimization problem which searches a small neighboring area of the original image (usually defined by $l_{1}$, $l_{2}$ or $l_{\infty}$ norm) for the existence of adversarial examples. If we denote an original image by $\bar{x}$ and the example with perturbation $\delta$ by $\hat{x} = \mathcal{F} (\bar{x} + \delta)$, then the process of searching adversarial examples can be formulated as follows \cite{goodfellow2014explaining}:
\begin{equation*}
\begin{aligned}
    & \underset{\delta}{\text{minimize}}
    & & || \hat{x} - \bar{x} || \\
    & \text{subject to}
    & & \mathcal{C} (\hat{x}) = z^{o} \\
    &
    & & \mathcal{F} (\bar{x} + \delta) \in \mathbb{R}_{[-1,1]}^{m}
\end{aligned}
\end{equation*}
The function $\mathcal{C}$ represents the classifier and outputs the pre-softmax logits based on input image. The function $\mathcal{F}$ projects the pixel value of any input image back to $\mathbb{R}_{[-1,1]}$ and ensures that the generated adversarial example is still a valid image. A perturbation is considered strong enough to fool the classifier if and only if $\mathcal{C} (\hat{x}) = z^{o}$, where $z^{o}$ is the objective pre-softmax logits designed by adversary. The global optimum of this problem corresponds to the strongest adversarial example for a given image. However, modern classifiers are highly non-linear, which makes it hard to solve the optimization problem in its original form, and hence each generator has its own approximation to make the optimization problem solvable. 

Table \ref{table:summary-notation} summarizes all the notations that we use throughout this paper. In the following, we describe the design approaches of several popular adversarial example generators that we consider in this work.

\begin{table}[tb]
    \begin{center}
    \begin{tabular}{ p{.25\linewidth} | p{.65\linewidth} }
    \hline \hline
    $\mathcal{L}, ~ \mathcal{L}_{\text{CLP}}, ~ \mathcal{L}_{\text{CLS}}$
    & loss function of NN classifier \\
    $\mathcal{F}$
    & regulation function for pixel value of generated example \\
    $z^{o}$
    & objective pre-softmax logits designed by adversary \\
    $l_{1}, ~ l_{2}, ~ l_{\infty}$
    & the 1st order, 2nd order and infinity order norm \\
    $\bar{x}, ~ \bar{X}; ~ \hat{x}, ~ \hat{X}; ~ x, ~ X$ 
    & original example; example with perturbation; their union \\
    $\bar{t}, ~ \bar{T}; ~ \hat{t}, ~ \hat{T}; ~ t, ~ T$ 
    & ground truth of $\bar{x}, ~ \bar{X}; ~ \hat{x}, ~ \hat{X}; ~ x, ~ X$ \\
    $\bar{z}, ~ \bar{Z}; ~ \hat{z}, ~ \hat{Z}; ~ z, ~ Z$ 
    & pre-softmax logits of $\bar{x}, ~ \bar{X}; ~ \hat{x}, ~ \hat{X}; ~ x, ~ X$ \\
    $\bar{s}, ~ \bar{S}; ~ \hat{s}, ~ \hat{S}; ~ s, ~ S$ 
    & source indicator of $\bar{x}, ~ \bar{X}; ~ \hat{x}, ~ \hat{X}; ~ x, ~ X$ \\
    $\delta$
    & perturbation \\
    $\mathcal{C}, ~ \mathcal{C}^{*}$
    & NN based classifier \\
    $\mathcal{D}, ~ \mathcal{D}^{*}$
    & NN based discriminator \\
    $J, ~ J'$
    & reward function of the minimax game \\
    $\Omega, ~ \Omega_{\mathcal{C}}, ~ \Omega_{\mathcal{D}}$
    & weight parameter in the NN model \\
    $\lambda, ~ \gamma$
    & trade-off hyper-parameters in CLP, CLS and GanDef \\
    \hline \hline
    \end{tabular}
    \end{center}
    \caption{Summary of Notations}
    \label{table:summary-notation}
\end{table}

\textbf{Fast Gradient Sign Method (FGSM)} is introduced by Goodfellow et. al in \cite{goodfellow2014explaining} as a single-step white-box adversarial example generator against NN image classifiers. This method tries to maximize the loss function value, $\mathcal{L}$, of NN classifier, $\mathcal{C}$, to find adversarial examples. The calculation of loss is usually defined as the difference between ground truth, $t$, and the softmax transformation, $f(z_{i}) = \frac{e^{z_{i}}}{\sum_{z_{j}} e^{z_{j}}}$, of pre-softmax logits.
%
\begin{equation*}
\begin{aligned}
    & \underset{\delta}{\text{maximize}}
    & & \mathcal{L} (\hat{z}=\mathcal{C}(\hat{x}), t) \\
    & \text{subject to}
    & & \mathcal{F} (\bar{x} + \delta) \in \mathbb{R}_{[-1,1]}^{m}
\end{aligned}
\end{equation*}
As a single-step generator, only one iteration of gradient ascent is executed to find adversarial examples. It simply generates examples with perturbation, $\hat{x}$, from original images, $\bar{x}$, by adding small perturbation, $\delta$, which changes each pixel value along the gradient direction of the loss function. Although running one iteration of gradient ascent algorithm can not guarantee finding a solution which is close enough to optimal one, empirical results show that adversarial examples from this generator can mislead Vanilla NN classifiers. Intuitively, FGSM runs faster than iterative generators at the cost of weaker adversarial examples. That is, the success rate of attack using the generated examples is relatively low due to the linear approximation of the loss function landscape. 

\textbf{Basic Iterative Method (BIM)} is introduced by Kurakin et. al in \cite{kurakin2016adversarial} as an iterative white-box adversarial example generator against NN image classifiers. BIM utilizes the same mathematical model as FGSM but runs the gradient ascent algorithm iteratively. In each iteration, BIM applies small perturbation and maps the perturbed image through the function $\mathcal{F}$. As a result, BIM approximates the loss function landscape by linear spline interpolation. Therefore, it generates stronger examples and achieves higher attack success rate than FGSM within the same neighboring area.

\textbf{Projected Gradient Descent (PGD)} is another iterative white-box adversarial example generator recently introduced by Madry et. al in \cite{madry2017towards}. Similar to BIM, PGD solves the same optimization problem iteratively with projected gradient descent algorithm. However, PGD randomly selects initial point within a limited area of the original image and repeats this several times to search adversarial example. Since the loss landscape has a surprisingly tractable structure \cite{madry2017towards}, PGD is shown experimentally to solve the optimization problem efficiently and the generated examples are stronger than those of BIM.

\subsection{Adversarial Example Defensive Methods}

Here, we categorize the design of defensive methods against adversarial examples into three major classes. The first is based on applying data augmentation and regularization during the training. The second class aims at adding protective shell on the target classifier, while the last class focuses on utilizing some adversarial examples to retrain the target classifier. In the following, we introduce representative examples from each of the above three approaches:

\textbf{Augmentation and Regularization} usually utilize synthetic data or regulate hidden states during training to enhance the test accuracy on adversarial examples. One of the early ideas in this direction is the defensive distillation. In the context of adversarial example defense, distillation is done by using the prediction score from original NN, which is usually called the teacher, as ground truth to train a smaller NN with different structure, usually called the student \cite{papernot2016distillation}\cite{papernot2017extending}. It has been shown that the calculated gradients from student model become very small or even reach zero and hence, can not be utilized by adversarial example generators \cite{papernot2017extending}. Examples of recent approaches under this category of defenses include Fortified Network \cite{lamb2018fortified} and Manifold Mixup \cite{verma2018manifold}. Fortified Network utilizes denoising autoencoder to regularize the hidden states. With this regularization, trained NN classifiers learn to mitigate the difference in hidden states between original and adversarial examples. Manifold Mixup also focus on hidden states but follows a different way. During training, Manifold Mixup uses interpolations of hidden states and logits instead of original training data to achieve regularization. However, this set of defenses is shown to be not very reliable as they are vulnerable to certain adversarial examples. For example, defensive distillation is vulnerable to Carlini attack \cite{carlini2016towards} and Manifold Mixup can only defend against single step attacks.

\textbf{Protective Shell} is a set of defensive methods designed to reject or reform adversarial examples. Meng et. al introduced an approach called MagNet \cite{meng2017magnet} which falls under this category. MagNet has two types of functional components; the detector and the reformer. Adversarial examples are either rejected by the detector or reformed by the reformer to clean up adversarial perturbations. Other recent approaches like \cite{liang2017detecting}, \cite{zhao2018detecting} and \cite{samangouei2018defense} also fall under this category and they are differentiated by the way they implement the protective shell. In \cite{liang2017detecting}, authors carefully inject adaptive noise to input images to break adversarial perturbations without significantly degrading classification accuracy. In \cite{zhao2018detecting}, a key based cryptography method is utilized to differentiate adversarial examples from original ones. In \cite{samangouei2018defense}, a generator is utilized to generate images that are similar to the inputs. By replacing the inputs with generated images, the approach shows good resistance to adversarial examples. The main limitation of the approaches under this category is the assumption that the shell is black-box to adversary, which turns to be inaccurate. For example, \cite{athalye2018obfuscated} presented  different ways to break this assumption.

\textbf{Adversarial Training} is based on the idea that adversarial examples can be considered as blind spots of the original training data \cite{xu2016automatically}. By retraining with samples of adversarial examples, the classifier learns perturbation patterns from adversarial examples and generalizes its prediction to account for such perturbations. In \cite{goodfellow2014explaining}, adversarial examples generated by FGSM are used for adversarial training of a NN classifier. The results show that the retrained classifier can correctly classify adversarial examples of this single step attack (FGSM). Later works in \cite{madry2017towards} and \cite{tramer2017ensemble} enhance the adversarial training process so that the trained models can defend not only single step attacks but also iterative attacks like BIM and PGD. A more recent work under this category \cite{kannan2018adversarial} introduces two zero knowledge adversarial training defenses. The defenses use Gaussian random noise for perturbations and include a penalty term based on pre-softmax logits, $z$. However, the design of penalty term is simple and not flexible enough to handle complex patterns in $z$. Our goal in this work is to design a flexible zero knowledge defense that handles $z$ in a more sophisticated way to achieve higher test accuracy on adversarial examples.

\section{ZK-GanDef: GAN based Zero Knowledge Adversarial Training Defense} \label{sec:defense}

In this section, we first introduce existing zero knowledge adversarial training defenses, then, we present the design and the algorithmic details of ZK-GanDef.
\begin{figure*}[tb]
\centering
\begin{minipage}[c]{.3\textwidth}
\centering
    \includegraphics[width=\linewidth]{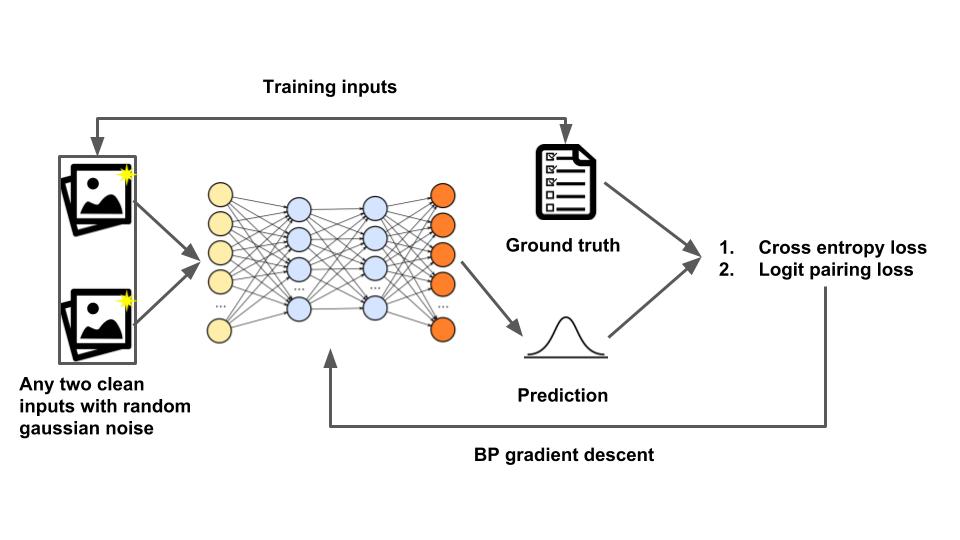}
    \subcaption{Clean Logit Pairing}
    \label{fig:clp-train}
\end{minipage}
\begin{minipage}[c]{.3\textwidth}
\centering
    \includegraphics[width=\linewidth]{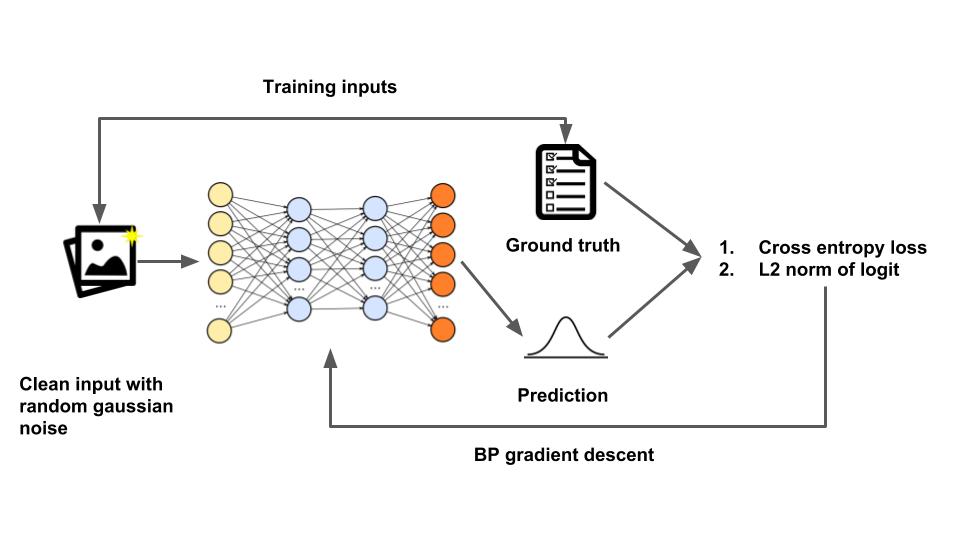}
    \subcaption{Clean Logit Squeezing}
    \label{fig:cls-train}
\end{minipage}
\begin{minipage}[c]{.3\textwidth}
\centering
    \includegraphics[width=\linewidth]{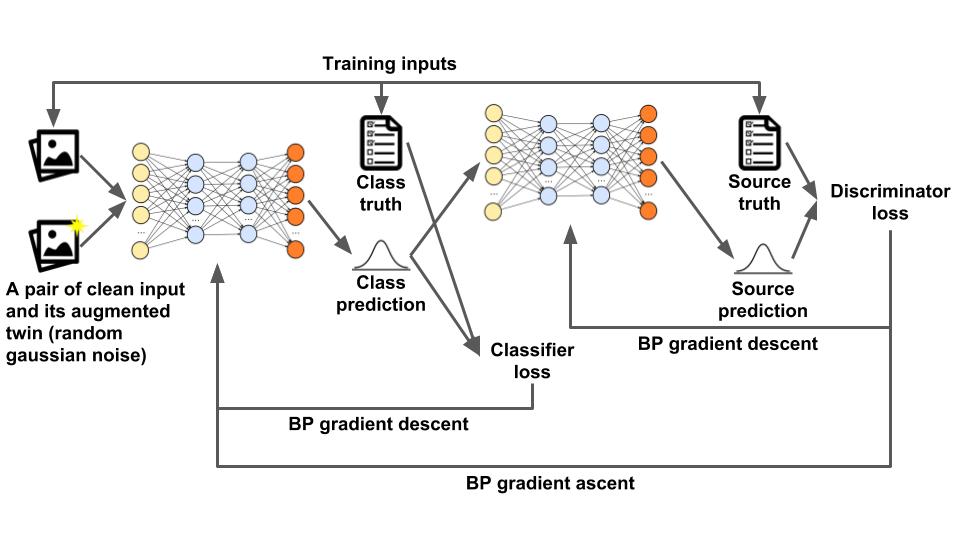}
    \subcaption{ZK-GanDef}
    \label{fig:zkg-train}
\end{minipage}
\caption{Training Procedure of Different Zero Knowledge Defenses}
\label{fig:zero-knowledge-defenses}
\end{figure*}

\subsection{Zero Knowledge Adversarial Training}

Recall that full knowledge adversarial training defenses retrain NN classifier with adversarial examples. Since adversarial examples are created by solving an optimization problem, its preparation consumes significant amount of computation, especially when iterative adversarial examples are utilized. Based on experiments in \cite{kannan2018adversarial}, generating adversarial examples on Imagenet dataset requires a cluster of GPU servers. To overcome this limitation, authors in \cite{kannan2018adversarial} also introduce two zero knowledge adversarial training defenses dubbed CLP and CLS. Instead of retraining with adversarial examples, these approaches retrain with examples perturbed with random Gaussian noise. The idea here is to speedup the training process by eliminating the computationally expensive step of adversarial examples generation. The caveat, however, is that since the retraining is performed with "fake" adversarial examples, the test accuracy against "true" adversarial examples degrades.

The training process of CLP is visualized in Figure \ref{fig:clp-train}. The retraining dataset consists of several pairs of randomly sampled original examples perturbed with random Gaussian noise. After the feed forward pass through the NN classifier, two different pre-softmax logits are generated. The differences between these pre-softmax logits and their corresponding ground truths are calculated as the first part of the total loss. The $l_{2}$ norm of the difference between these two pre-softmax logits is also calculated and used as the second part in the total loss. Based on the total loss, the weights, $\Omega$, of the NN classifier are updated by gradient descent algorithm and back propagation. The training loss of CLP can be summarized as follows:
\begin{equation*}
\begin{aligned}
    & \mathcal{L}_{\text{CLP}}(\mathcal{C}) = 
    & \mathcal{L}(\hat{z}_{1} = \mathcal{C}(\hat{x}_{1}), \hat{t}_{1})
    & + \mathcal{L}(\hat{z}_{2} = \mathcal{C}(\hat{x}_{2}), \hat{t}_{2}) \\
    & & & + \lambda l_{2}(\mathcal{C}(\hat{x}_{1}) - \mathcal{C}(\hat{x}_{2}))
\end{aligned}
\end{equation*}

The training process of the other zero knowledge approach, CLS, is shown in Figure \ref{fig:cls-train}. Similar to CLP, CLS retrains with examples perturbed with random Gaussian noise. However, instead of using pairs of inputs, CLS uses individual inputs to the NN classifier in the forward pass. The first term of the total loss in CLS is still calculated by a predefined loss function of pre-softmax logits and the corresponding ground truths. Different from the CLP, CLS directly calculates the $l_{2}$ norm of pre-softmax logits as the second term in its total loss. Thereafter, it follows the same training process with gradient descent algorithm and back propagation to update the weights, $\Omega$, in the NN classifier. The loss function of CLS is as follows:

\begin{equation*}
\begin{aligned}
    & \mathcal{L}_{\text{CLS}}(\mathcal{C}) = 
    & \mathcal{L}(\hat{z} = \mathcal{C}(\hat{x}), \hat{t})
    & + \lambda l_{2}(\mathcal{C}(\hat{x}))
\end{aligned}
\end{equation*}

The hypothesis behind CLP and CLS is that abnormal large values in pre-softmax logits are signals of adversarial examples. Therefore, they both add penalty term to the loss function during the training in order to prevent such over confident predictions. Although the penalty terms are different, both defenses encourage the NN classifier to output small and smooth pre-softmax logits. 

\subsection{Design of ZK-GanDef}

As mentioned in the previous subsection, CLP and CLS try to prevent overconfident predictions by penalizing high pre-softmax logits. However, the penalty terms used are oversimplified and do not utilize other valuable information contained in the pre-softmax logits. This results, as we see in the evaluation section, in poor accuracy on complex datasets. On the other hand, our ZK-GanDef is designed to better utilize the rich information available in the pre-softmax logits. As Figure \ref{fig:zkg-train} shows, Zk-Gandef comprises a classifier and a discriminator. The input to the classifier includes both original images and randomly perturbed examples. It has been shown in transfer learning, \cite{goodfellow2016deep}, that the pre-softmax logits output of the classifier relates to the extracted features from its inputs. Therefore, we use a discriminator to identify whether the logit output of the classifier belongs to an original image or a perturbed example. The intuition here is that the features extracted by a Vanilla NN classifier from perturbed examples will contain some kind of perturbations, and hence can be recognized by a trained discriminator.

In this work, we envision that the classifier could be seen as a generator that generates pre-softmax logits based on selected features from inputs. Then, the classifier and the discriminator engage in a minmax game, which is also known as \textit{Generative Adversarial Net} (GAN) \cite{goodfellow2016deep}. In this minimax game, the discriminator tries to make perfect prediction about the source of inputs (original or perturbed). At the same time, the classifier tries to correctly classify inputs as well as hide the source information from the discriminator. This process trains a classifier which makes prediction based on perturbation invariant features from inputs, as well as a discriminator which can identify whether the features used by the fellow classifier contain any perturbations. \textcolor{black}{Through training in this competition game, the feature learning in the classifier is regulated by the discriminator and it finally leads to defense against adversarial examples.}

Compared with the CLP and CLS, ZK-GanDef has a more sophisticated way of utilizing pre-softmax logits. Instead of encouraging the NN classifier to make small and smooth logits, ZK-GanDef aims at differentiating the latent pattern of logits between original images and examples with perturbations. Therefore, the NN classifier in ZK-GanDef is encouraged to select perturbation invariant features, which enhance its test accuracy of adversarial examples on complex datasets. \textcolor{black}{It is worth to mention that an example with Gaussian perturbation is not necessary to be an adversarial example. However, results in \cite{kannan2018adversarial} show that defenses against adversarial examples can be built by training against examples with Gaussian perturbation. Our method is built upon this empirical conclusion.}

\subsection{ZK-GanDef Training Algorithm}

Given the training data pair $\langle x, t \rangle$, where $x \in \cup (\bar{X}, \hat{X})$, we try to find a classification function $\mathcal{C}$, which uses $x$ to give a proper pre-softmax logits $z$ corresponding to $t$. The goal is to train the classifier in ZK-GanDef to model the conditional probability $q_{C}(z|x)$ with only perturbation invariant features. To achieve this, we employ another NN and call it a discriminator $\mathcal{D}$. $\mathcal{D}$ uses the pre-softmax logits $z$ from $\mathcal{C}$ as inputs and predicts whether the input image to $\mathcal{C}$ was $\bar{x}$ or $\hat{x}$. This process can be performed by maximizing the conditional probability $q_{D}(s|z)$, where $s$ is a Boolean variable indicating whether the input to $\mathcal{C}$ was original or randomly perturbed image. The combined minmax problem of the classifier and the discriminator is formulated as:
\begin{equation*}
\begin{aligned}
    & \underset{\mathcal{C}}{\text{min}} ~ \underset{\mathcal{D}}{\text{max}} ~ J(\mathcal{C}, \mathcal{D})
\end{aligned}
\end{equation*}
\begin{equation*}
\begin{aligned}
    & \text{where} ~~ J(\mathcal{C}, \mathcal{D}) = \\
    & \underset{x \sim X, t \sim T}{\mathbb{E}} \{- log [q_{C}(z|x)]\}
    - \underset{z \sim Z, s \sim S}{\mathbb{E}} \{- log [q_{D}(s|z=\mathcal{C}(x))]\}
\end{aligned}
\end{equation*}

In words, the training process of the classifier ($\mathcal{C}$) tries to minimize the log likelihood of predicting $s$ from $z$, while maximizing the log likelihood of predicting $z$ from $x$. At the same time, the goal of the discriminator ($\mathcal{D}$) is to maximize the log likelihood of predicting $s$ from $z$. Recall that, similar to CLP and CLS \cite{kannan2018adversarial}, ZK-GanDef uses inputs ($x$) perturbed with random Gaussian noise as an approximation of true adversarial examples.

The pseudocode for training of ZK-GanDef is shown in Algorithm \ref{algorithm:defense}. During the sampling in lines \ref{alg2:sample1} and \ref{alg2:sample2}, a number (predefined by user) of examples is evenly sampled from original images $\bar{X}$ and examples with Gaussian perturbations $\hat{X}$ to form a training batch. In lines \ref{alg2:freeze1} and \ref{alg2:freeze2}, the weight parameters in the classifier (discriminator) are frozen before updating the weight parameters in the discriminator (classifier). Finally, in lines \ref{alg2:update1} and \ref{alg2:update2}, the weight parameters are updated through the stochastic gradient descent algorithm. In this algorithm, we iteratively update the classifier and the discriminator one at a time to emulate the proposed minimax game.

\begin{algorithm} \caption{Training ZK-GanDef} \label{algorithm:defense}
\begin{algorithmic}[1]
\Require training data $X$, ground truth $T$, classifier $\mathcal{C}$, discriminator $\mathcal{D}$
\Ensure classifier $\mathcal{C}$, discriminator $\mathcal{D}$
\State Initialize weight parameters $\Omega$ in both classifier and discriminator
\For {the global training iterations}
    \For {the discriminator training iterations}
        \State Sample a batch of training pair, $\langle x, t \rangle$ \label{alg2:sample1}
        \State Generate a batch of boolean indicator, $s$, corresponding to training inputs \label{alg2:prepare}
        \State Fix $\Omega_{\mathcal{C}}$ in classifier $\mathcal{C}$ \label{alg2:freeze1}
        \State Update $\Omega_{\mathcal{D}}$ in discriminator $\mathcal{D}$ \label{alg2:update1}
    \EndFor
    \State Sample a batch of training pair, $\langle x, t \rangle$ \label{alg2:sample2}
    \State Generate a batch of boolean indicator, $s$, corresponding to training inputs
    \State Fix $\Omega_{\mathcal{D}}$ in discriminator $\mathcal{D}$ \label{alg2:freeze2}
    \State Update $\Omega_{\mathcal{C}}$ in classifier $\mathcal{C}$ \label{alg2:update2}
\EndFor
\end{algorithmic}
\end{algorithm}

\subsection{Theoretical Analysis}

Given that $J$ is a combination of the log likelihood of $Z|X$ and $S|Z$, \textcolor{black}{we provide a mathematical intuition here that the solution of the minimax game is a classifier which makes correct predictions based on perturbation invariant features.} It is worth noting that our analysis is conducted in a non-parametric setting, which means that the classifier and the discriminator have enough capacity to model any distribution.

\begin{prop} \label{prop:1}
    If there exists a solution $(\mathcal{C}^{*}, \mathcal{D}^{*})$ for the aforementioned minmax game $J$ such that $J(\mathcal{C}^{*}, \mathcal{D}^{*}) = H(Z|X) - H(S)$, then $\mathcal{C}^{*}$ is an optimal classifier which correctly classifies adversarial inputs.
\end{prop}

\begin{proof} For any fixed classification model $\mathcal{C}$, the optimal discriminator can be formulated as
\begin{equation*}
\begin{aligned}
    \mathcal{D}^{*} &= \text{arg} ~ \underset{\mathcal{D}}{\text{max}} ~ J(\mathcal{C}, \mathcal{D}) \\
    &= \text{arg} ~ \underset{\mathcal{D}}{\text{min}} ~ \underset{z \sim Z, s \sim S}{\mathbb{E}} \{- log [q_{D}(s|z=\mathcal{C}(x))]\}
\end{aligned}
\end{equation*}
In this case, the discriminator can perfectly model the conditional distribution and we have $q_{D}(s|z=\mathcal{C}(x)) = p(s|z=\mathcal{C}(x))$ for all $z$ and all $s$. Therefore, we can rewrite $J$ with optimal discriminator as $J'$ and denote the second half of $J$ as a conditional entropy $H(S|Z)$
\begin{equation*}
\begin{aligned}
    & J'(\mathcal{C}) = \underset{x \sim X, t \sim T}{\mathbb{E}} \{- log [q_{C}(z|x)]\} - H(S|Z)
\end{aligned}
\end{equation*}
For the optimal classification model, the goal is to achieve the conditional probability $q_{C}(z|x) = p(z|x)$ since $z$ can determine $t$ by taking softmax transformation. Therefore, the first part of $J'(\mathcal{C})$ (the expectation) is larger than or equal to $H(Z|X)$. Combined with the basic property of conditional entropy that $H(S|Z) \leq H(S)$, we can get the following lower bound of $J$ with optimal classifier and discriminator
\begin{equation*}
\begin{aligned}
    & J(\mathcal{C}^{*}, \mathcal{D}^{*}) \geq H(Z|X) - H(S|Z) \geq H(Z|X) - H(S)
\end{aligned}
\end{equation*}
This equality holds when the following two conditions are satisfied:
\begin{itemize}
    \item The classifier perfectly models the conditional distribution of $z$ given $x$, $q_{C}(z|x) = p(z|x)$, which means that $\mathcal{C}^{*}$ is an optimal classifier.
    \item The $S$ and $Z$ are independent, $H(S|Z) = H(S)$, which means that perturbations do not affect pre-softmax logits.
\end{itemize}
\end{proof}

In practice, the assumption of unlimited capacity in classifier and discriminator may not hold and it would be hard or even impossible to build an optimal classifier which outputs pre-softmax logits independent from adversarial perturbations. Therefore, we introduce a trade-off hyper-parameter $\gamma$ into the minimax function as follows:
\begin{equation*}
\begin{aligned}
    & \underset{x \sim X, t \sim T}{\mathbb{E}} \{- log [q_{C}(z|x)]\}  - \gamma \underset{z \sim Z, s \sim S}{\mathbb{E}} \{- log [q_{D}(s|z=\mathcal{C}(x))]\}
\end{aligned}
\end{equation*}
When $\gamma = 0$, ZK-GanDef is the same as traditional adversarial training. When $\gamma$ increases, the discriminator becomes more and more sensitive to information of $s$ contained in pre-softmax logits, $z$.

Based on previous proof, ZK-GanDef achieves feature selection through the design of minimax game with discriminator. By selecting perturbation invariant features, the classifier could defend against adversarial examples since they are also combinations of original image and adversarial perturbations.

\section{Evaluation Settings} \label{sec:experiment}

\begin{figure}[tb]
\centering
\begin{minipage}[c]{.4\textwidth}
\centering
    \includegraphics[width=\linewidth]{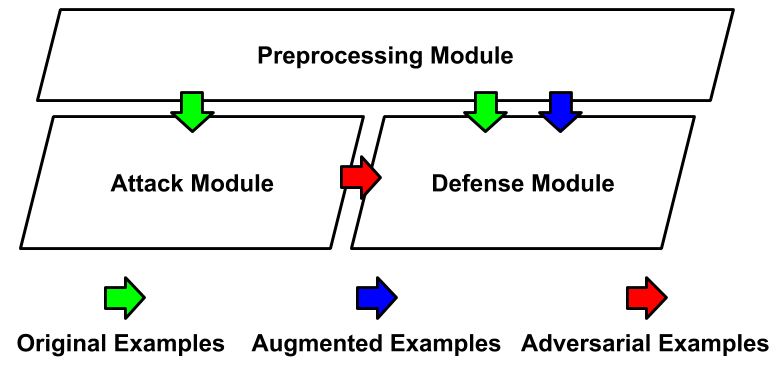}
    \caption{Evaluation Framework}
    \label{fig:framework}
\end{minipage}
\end{figure}

This section presents the framework that we use to evaluate our defensive method, ZK-GanDef, under different popular adversarial example generators and compare it with other state-of-the-art zero knowledge adversarial training defenses. Figure \ref{fig:framework} depicts the main components of this framework, which include: (1) Preprocessing module, (2) Attack module and (3) Defense module. Different adversarial example generators and defensive methods could be used as plug-ins to Attack and Defense modules respectively, to form different test scenarios.

In the following subsections, we present the datasets utilized, the detailed description of each module, and a summary of the evaluation metrics used.

\subsection{Datasets}

During our evaluations, the following datasets are utilized:
\begin{itemize}
    \item MNIST: Contains a total of 70K images and their labels. Each one is a $28 \times 28$ pixel, gray scale labeled image of handwritten digit.
    \item Fashion-MNIST: Contains a total of 70K images and their labels. Each one is a $28 \times 28$ pixel, gray scale labeled image of different kinds of clothes.
    \item CIFAR10: Contains a total of 60K images and their labels. Each one is a $32 \times 32$ pixel, RGB labeled image of animal or vehicle.
\end{itemize}
The images in each dataset are evenly labeled into 10 different classes. Although Fashion-MNIST has exactly the same image size as MNIST, images in Fashion-MNIST have far more details than images from MNIST.

\subsection{Preprocessing Module}

Preprocessing module involves the following operations:
\begin{itemize}
    \item Scaling: Gray scale images use one integer to represent each of their pixels, while RGB images use three different integers (each between 0 and 255) to represent each of their pixels. To simplify the process of finding adversarial examples and to be consistent with the related work, scaling is used to map pixel representations from discrete integers in the range $\mathbb{Z}_{[0,255]}$ into real numbers in the range $\mathbb{R}_{[-1,1]}$.
    \item Separation: This operation is used to split each input dataset into two groups: training-dataset and testing-dataset. The training dataset is used to train the supervised machine learning models which are the different NN classifiers in this work, while the testing dataset is used by the attack module to generate adversarial examples in order to evaluate the NN classifier under test. The detailed separation plans are: (1) the 70K MNIST and Fashion-MNIST images are randomly separated into 60K training and 10K testing images, respectively and (2) the 60K CIFAR10 images are randomly separated into 50K training and 10K testing images.
    \item Augmentation: This operation is used to generate augmented examples for different zero knowledge adversarial training methods. Based on the description in \cite{kannan2018adversarial} and our communication with its authors, we keep the same augmentation which is adding a Gaussian perturbation with mean $\mu = 0$ and standard deviation $\sigma = 1$. The Gaussian perturbation used in this work is not guaranteed to be the optimal choice and we keep the detailed comparison of different augmentation methods as future work.
\end{itemize}

\subsection{Attack Module}

The attack module implements three popular adversarial example generators, the FGSM \cite{goodfellow2014explaining}, the BIM \cite{kurakin2016adversarial} and the PGD \cite{madry2017towards}. \textcolor{black}{As we mention in the previous section, all adversarial example generators are utilized under the white-box scenario. Moreover, each original example has its own corresponding adversarial counterparts (FGSM, BIM, PGD).} Adversarial examples from same dataset share same maximum $l_{\infty}$ perturbation limits which are 0.6 in MNIST \& Fashion-MNIST and 0.06 in CIFAR10. For the BIM, we also limit the per step perturbation to 0.1 in MNIST \& Fashion-MNIST and 0.016 in CIFAR10. Finally, for the PGD, we run generation algorithm 40 iteration with 0.02 per step perturbation on MNIST \& Fashion-MNIST and 20 iteration with 0.016 per step perturbation on CIFAR10. To ensure the quality of the adversarial example generators, we choose the standard python library, CleverHans \cite{papernot2018cleverhans}, which is adopted by the community.

\subsection{Defense Module}

This module implements the Vanilla NN classifiers as well as the different defense methods that we evaluate in this work. For the same dataset, different defense methods share the same structure of the classifier as that of the Vanilla. \textcolor{black}{Hyper-parameters of defenses we compare with are the exact ones used in their original papers. Our ZK-GanDef is tuned by line search to find a suitable hyper-parameter setting.}

\subsubsection{Vanilla Classifier}

For each dataset, we use as a baseline a NN classifier with no defenses, which is also referred to as the Vanilla classifier. We select different Vanilla classifiers for each dataset. The structure of the Vanilla classifier used in MNIST and Fashion-MNIST dataset is LeNet \cite{madry2017towards}. For the CIFAR10 dataset, we use the allCNN based classifier \cite{springenberg2014striving}. Due to space limitations, the detailed NN structure and training settings are not listed.

\subsubsection{Zero Knowledge Defenses}

We implement here three different approaches: (1) a classifier trained with CLP \cite{kannan2018adversarial}, (2) a classifier trained with CLS \cite{kannan2018adversarial}, and (3) a classifier trained with ZK-GanDef. As Figures \ref{fig:clp-train} and \ref{fig:cls-train} show, CLP and CLS train only with randomly perturbed examples. On the other hand, ZK-GanDef (Figure \ref{fig:zkg-train}) trains with both original and randomly perturbed examples. We note also that the structure of the discriminator in ZK-GanDef (Table \ref{table:discriminator-structure}) does not change with different datasets. Training of the discriminator utilizes the Adam optimizer \cite{kingma2014adam} with 0.001 learning rate.

\subsubsection{Full Knowledge Defenses}

We implement here three of the full knowledge defenses: (1) a classifier trained with original and FGSM examples (FGSM-Adv), (2) a classifier trained with original and PGD examples (PGD-Adv), and (3) a classifier trained with original and PGD examples through GAN based training (PGD-GanDef). Among them, PGD-Adv is the state-of-the-art full knowledge adversarial training defense.

\begin{table}[tb]
    \begin{center}
    \begin{tabular}{c | c | c | c | c }
    \hline \hline
    Layer & Kernel Size & Strides & Padding & Activation \\
    \hline \hline
    Dense & 32 & - & - & ReLU \\
    \hline
    Dense & 64 & - & - & ReLU \\
    \hline
    Dense & 32 & - & - & ReLU \\
    \hline
    Dense & 1 & - & - & Sigmoid \\
    \hline \hline
    \end{tabular}
    \end{center}
    \caption{Discriminator Structure}
    \label{table:discriminator-structure}
\end{table}

\subsection{Evaluation Metrics}

The overall classifier performance is captured by the \textit{test accuracy} metric, which is defined as the ratio of the total number of tested images minus the number of failed tests to the total number of tested images (both original and adversarial). 
\begin{equation*}
\begin{aligned}
    \text{test accuracy} \equiv \frac{\text{total \# of test examples}~ - ~\text{\# of failed tests}}{\text{total \# of test examples}}
\end{aligned}
\end{equation*}
A test is considered failed when: (1) original example is missclassifed, (2) original example is rejected, or (3) adversarial example is accepted with incorrect classification. To be more precise during evaluation, we separately compute the test accuracy on original and adversarial examples. When a defensive method tries to maximize classifier's capability to identify adversarial examples, the classifier may reject or missclassify more original examples than the corresponding Vanilla classifier. The trade-off between correctly classifying original and adversarial examples is the same as the trade-off between \textit{true positive rate} and \textit{true negative rate} in machine learning.

The other important metric to evaluate defense approaches is the training time it takes to build the model. As mentioned earlier, a significant amount of computation is consumed to generate the adversarial examples for full knowledge adversarial training. The two main contributing factors to the training time are: (1) the structure of the classifier (number of layers and parameters) and (2) the searching algorithm of adversarial examples (e.g., single-step vs. iterative approaches). The goal is to minimize the training time while maintaining acceptable test accuracy.

\section{Experimental Results} \label{sec:results}

\begin{figure*}[p]
\centering
\begin{minipage}[c]{.24\textwidth}
\centering
    \includegraphics[width=\linewidth]{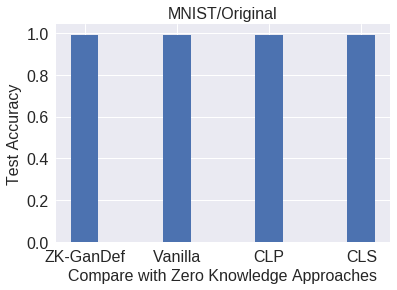}
    \label{fig:mnist-zero-ori}
\end{minipage}
\begin{minipage}[c]{.24\textwidth}
\centering
    \includegraphics[width=\linewidth]{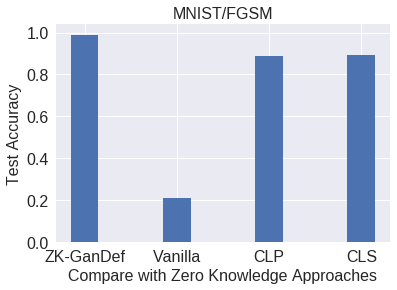}
    \label{fig:mnist-zero-fgsm}
\end{minipage}
\begin{minipage}[c]{.24\textwidth}
\centering
    \includegraphics[width=\linewidth]{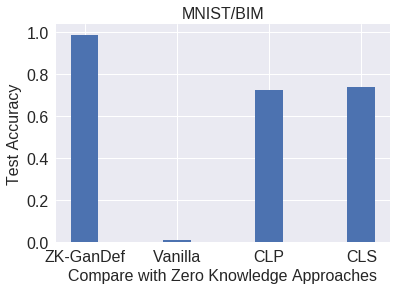}
    \label{fig:mnist-zero-bim}
\end{minipage}
\begin{minipage}[c]{.24\textwidth}
\centering
    \includegraphics[width=\linewidth]{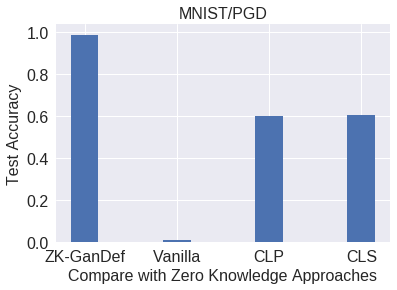}
    \label{fig:mnist-zero-pgd}
\end{minipage}
\begin{minipage}[c]{.24\textwidth}
\centering
    \includegraphics[width=\linewidth]{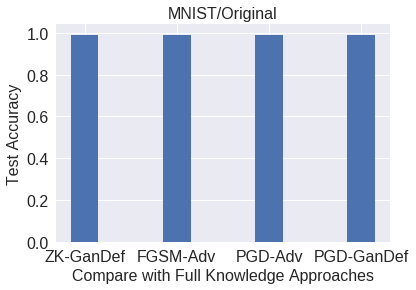}
    \label{fig:mnist-full-ori}
\end{minipage}
\begin{minipage}[c]{.24\textwidth}
\centering
    \includegraphics[width=\linewidth]{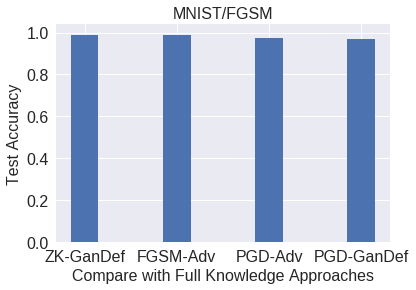}
    \label{fig:mnist-full-fgsm}
\end{minipage}
\begin{minipage}[c]{.24\textwidth}
\centering
    \includegraphics[width=\linewidth]{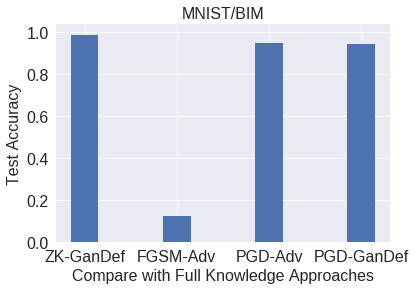}
    \label{fig:mnist-full-bim}
\end{minipage}
\begin{minipage}[c]{.24\textwidth}
\centering
    \includegraphics[width=\linewidth]{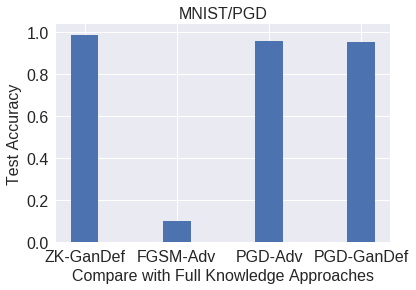}
    \label{fig:mnist-full-pgd}
\end{minipage}
\begin{minipage}[c]{.24\textwidth}
\centering
    \includegraphics[width=\linewidth]{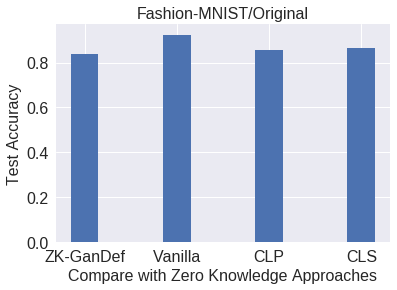}
    \label{fig:fmnist-zero-ori}
\end{minipage}
\begin{minipage}[c]{.24\textwidth}
\centering
    \includegraphics[width=\linewidth]{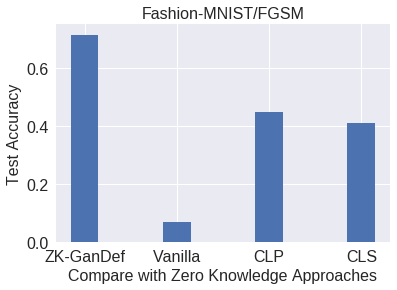}
    \label{fig:fmnist-zero-fgsm}
\end{minipage}
\begin{minipage}[c]{.24\textwidth}
\centering
    \includegraphics[width=\linewidth]{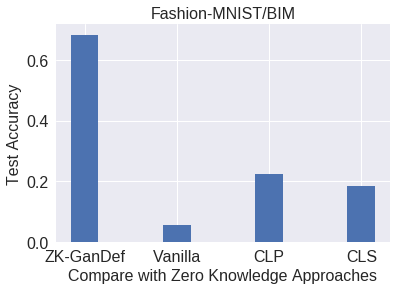}
    \label{fig:fmnist-zero-bim}
\end{minipage}
\begin{minipage}[c]{.24\textwidth}
\centering
    \includegraphics[width=\linewidth]{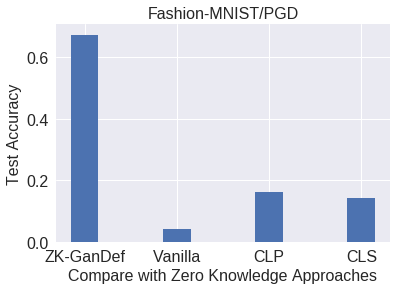}
    \label{fig:fmnist-zero-pgd}
\end{minipage}
\begin{minipage}[c]{.24\textwidth}
\centering
    \includegraphics[width=\linewidth]{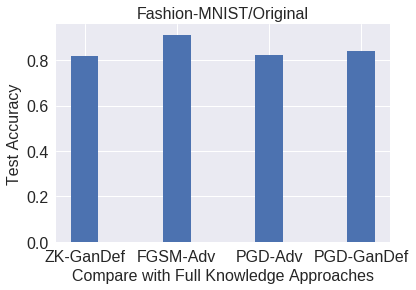}
    \label{fig:fmnist-full-ori}
\end{minipage}
\begin{minipage}[c]{.24\textwidth}
\centering
    \includegraphics[width=\linewidth]{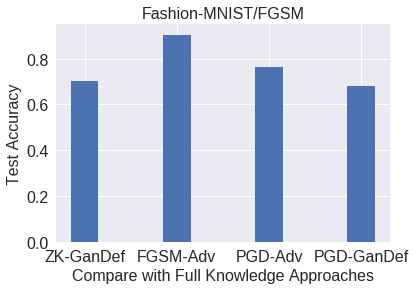}
    \label{fig:fmnist-full-fgsm}
\end{minipage}
\begin{minipage}[c]{.24\textwidth}
\centering
    \includegraphics[width=\linewidth]{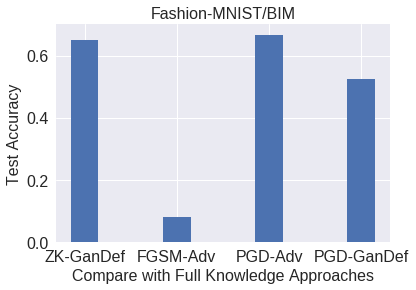}
    \label{fig:fmnist-full-bim}
\end{minipage}
\begin{minipage}[c]{.24\textwidth}
\centering
    \includegraphics[width=\linewidth]{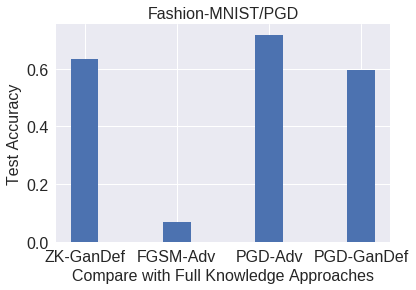}
    \label{fig:fmnist-full-pgd}
\end{minipage}
\begin{minipage}[c]{.24\textwidth}
\centering
    \includegraphics[width=\linewidth]{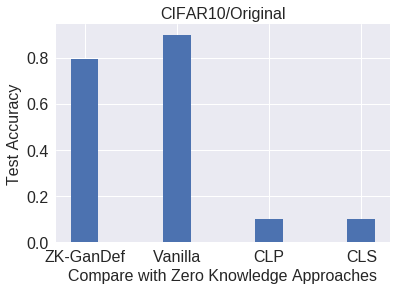}
    \label{fig:cifar-zero-ori}
\end{minipage}
\begin{minipage}[c]{.24\textwidth}
\centering
    \includegraphics[width=\linewidth]{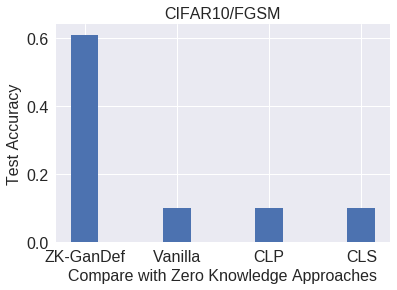}
    \label{fig:cifar-zero-fgsm}
\end{minipage}
\begin{minipage}[c]{.24\textwidth}
\centering
    \includegraphics[width=\linewidth]{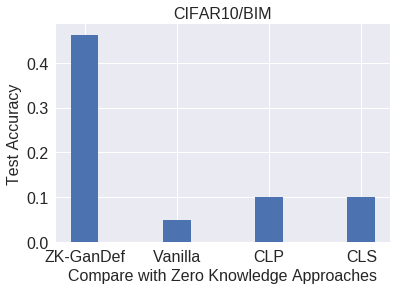}
    \label{fig:cifar-zero-bim}
\end{minipage}
\begin{minipage}[c]{.24\textwidth}
\centering
    \includegraphics[width=\linewidth]{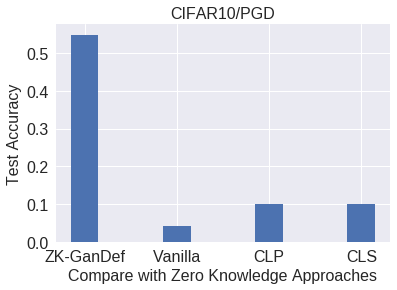}
    \label{fig:cifar-zero-pgd}
\end{minipage}
\begin{minipage}[c]{.24\textwidth}
\centering
    \includegraphics[width=\linewidth]{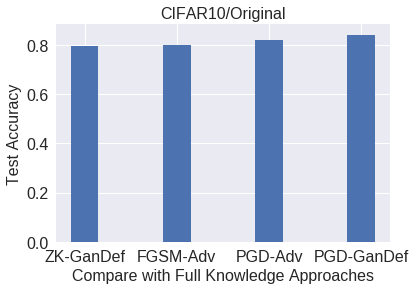}
    \label{fig:cifar-full-ori}
\end{minipage}
\begin{minipage}[c]{.24\textwidth}
\centering
    \includegraphics[width=\linewidth]{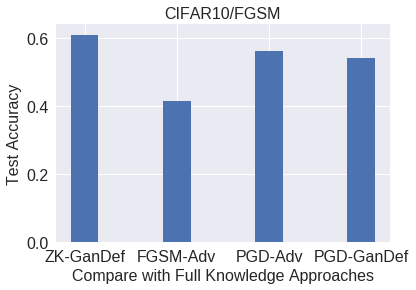}
    \label{fig:cifar-full-fgsm}
\end{minipage}
\begin{minipage}[c]{.24\textwidth}
\centering
    \includegraphics[width=\linewidth]{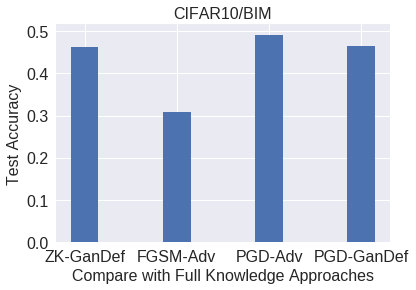}
    \label{fig:cifar-full-bim}
\end{minipage}
\begin{minipage}[c]{.24\textwidth}
\centering
    \includegraphics[width=\linewidth]{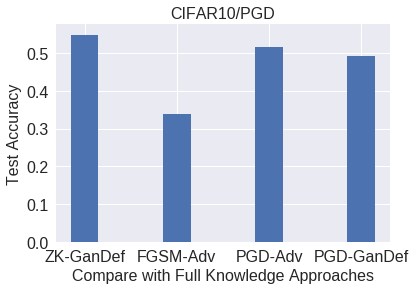}
    \label{fig:cifar-full-pgd}
\end{minipage}
\caption{Test Accuracy on Different Examples (\textit{In the $1^{st}$ and $2^{nd}$ rows, the results on MNIST dataset are presented. In the $3^{rd}$ and $4^{th}$ rows, the results on Fashion-MNIST dataset are presented. In the $5^{th}$ and $6^{th}$ rows, the results on CIFAR10 dataset are presented. The results in odd number rows compare the proposed ZK-GanDef with Vanilla as well as existing zero knowledge methods, CLP and CLS. The results in even number rows compare the proposed ZK-GanDef with full knowledge defenses which include FGSM-Adv, PGD-Adv and PGD-GanDef.})}
\label{fig:test-acc}
\end{figure*}
\begin{table*}[tb]
    \begin{center}
    \begin{tabular}{ c | c  c  c  c | c  c  c  c | c  c  c  c }
    \hline \hline
    & \multicolumn{4}{c |}{MNIST} & \multicolumn{4}{c |}{Fashion-MNIST} & \multicolumn{4}{c}{CIFAR10} \\
    & Original & FGSM & BIM & PGD & Original & FGSM & BIM & PGD & Original & FGSM & BIM & PGD \\
    \hline
    Vanilla & 98.92\% & 21.01\% & 1.00\% & 0.77\% & 92.43\% & 7.01\% & 5.62\% & 4.06\% & 89.92\% & 9.97\% & 4.93\% & 4.06\% \\
    CLP & 99.13\% & 88.70\% & 72.61\% & 59.93\% & 85.65\% & 44.78\% & 22.30\% & 16.14\% & 10.00\%\footnotemark & 10.00\%\footnotemark[\value{footnote}] & 10.00\%\footnotemark[\value{footnote}] & 10.00\%\footnotemark[\value{footnote}] \\
    CLS & 99.24\% & 89.29\% & 73.84\% & 60.63\% & 86.37\% & 41.14\% & 18.55\% & 14.17\% & 10.00\%\footnotemark[\value{footnote}] & 10.00\%\footnotemark[\value{footnote}] & 10.00\%\footnotemark[\value{footnote}] & 10.00\%\footnotemark[\value{footnote}] \\
    ZK-GanDef & 98.95\% & 98.97\% & 98.89\% & 98.71\% & 81.95\% & 70.19\% & 64.97\% & 63.34\% & 79.33\% & 60.91\% & 46.27\% & 54.85\% \\
    FGSM-Adv & 99.07\% & 98.79\% & 12.24\% & 9.73\% & 91.17\% & 90.48\% & 7.97\% & 6.81\% & 79.88\% & 41.53\% & 30.74\% & 33.86\% \\
    PGD-Adv & 99.15\% & 97.60\% & 94.75\% & 95.60\% & 82.33\% & 76.42\% & 66.72\% & 71.80\% & 82.06\% & 56.18\% & 49.21\% & 51.51\% \\
    PGD-GanDef & 99.10\% & 96.85\% & 94.28\% & 95.31\% & 84.09\% & 68.19\% & 52.35\% & 59.51\% & 84.05\% & 54.14\% & 46.64\% & 49.21\% \\
    \hline \hline
    \end{tabular}
    \end{center}
    \caption{Test Accuracy on Different Examples (\textit{The left column shows the test results on MNIST dataset. The middle column shows the test results on Fashion-MNIST dataset. The right column shows the test results on CIFAR10 dataset.})}
    \label{table:summary-test-accuracy}
    \vspace{-3mm}
\end{table*}

In this section, we present comparative evaluation results of the ZK-GanDef with other state-of-the-art zero knowledge as well as full knowledge adversarial training defenses introduced previously. The evaluation results are summarized in three subsections. In the first subsection, we provide comparative evaluation of ZK-GanDef with other zero knowledge and full knowledge adversarial training defenses on classifying original and different types of adversarial examples. Then, we compare the computational consumption of ZK-GanDef with other full knowledge adversarial training defenses in terms of training time per epoch. In the third subsection, we analyze the convergence issues of CLP and CLS on CIFAR10 dataset.

\subsection{Test Accuracy on Different Examples}

In this subsection, we show the test accuracy of the Vanilla classifier and the classifiers with defenses against different types of examples. As mentioned earlier, the experiments are conducted on MNIST, Fashion-MNIST and CIFAR10 datasets. For each dataset, a total of 28 different results are calculated. These results span all possible pairs of 7 different classifiers (Vanilla, CLP, CLS, ZK-GanDef, FGSM-Adv, PGD-Adv and PGD-GanDef) and 4 different kinds of examples (original, FGSM, BIM and PGD). All the validation results are presented in Figure \ref{fig:test-acc} and detailed in Table \ref{table:summary-test-accuracy}.

\subsubsection{On Original Examples}

In Figure \ref{fig:test-acc}, we first focus on the results presented in the first column sub-figures. These results represent the test accuracy on original examples from different datasets. As a baseline, the Vanilla classifier achieves 98.92\% test accuracy on MNIST, 92.43\% test accuracy on Fashion-MNIST and 89.92\% test accuracy on CIFAR10. These results are consistent with the benchmark ones presented in \cite{benchmark-list}.

We then evaluate the test accuracy of the three zero knowledge defenses (CLP, CLS and ZK-GanDef) on different datasets. On MNIST dataset, their test accuracy on original examples is at the same level as that of the Vanilla classifier. The detailed results from Table \ref{table:summary-test-accuracy} show that the difference in test accuracy among the defenses is within 0.5\%, which is small enough to be ignored. On Fashion-MNIST dataset, the test accuracy of CLP and CLS is 5\% higher than that of ZK-GanDef on original examples. Moreover, the test accuracy of all zero knowledge approaches is (6\% to 11\%) lower than that of the Vanilla classifier. This small degeneration is a result of tuning the model to enhance test accuracy on adversarial examples. On CIFAR10 dataset, CLP and CLS have a significantly lower test accuracy compared with the Vanilla classifier and ZK-GanDef. This is because the classifiers with the CLP and CLS methods do not converge at the beginning of the training. A detailed study of this phenomenon is provided in the following subsection.
\footnotetext{On CIFAR10 dataset, CLP and CLS have convergence issues during training and hence the classifier is making random guessing. A detailed study of this issue is provided in a following subsection.}

Finally, we conduct the same evaluation with full knowledge adversarial training defenses and perform comparison with the proposed ZK-GanDef. On MNIST dataset, all full knowledge defenses and ZK-GanDef achieve the same level of test accuracy as that of the Vanilla classifier. On Fashion-MNIST dataset, FGSM-Adv achieves similar test accuracy on original examples to that of the Vanilla classifier, while ZK-GanDef, PGD-Adv and PGD-GanDef have about 10\% to 12\% degeneration from that of the Vanilla classifier. On CIFAR10 dataset, ZK-GanDef performance is similar to that of full knowledge defenses and their test accuracy on original examples are 6\% to 10\% lower than that of the Vanilla classifier, respectively.  To enhance test accuracy on adversarial examples, the decision boundary of the classifier becomes complex with more curves, which causes the degeneration on classifying original examples compared to the Vanilla classifier \cite{madry2017towards}.

\subsubsection{On Single-step Adversarial Examples}

We discuss here the accuracy results on FGSM examples, which are depicted on sub-figures on the second column of Figure \ref{fig:test-acc}. Intuitively, the Vanilla classifier has poor performance on these single-step adversarial examples, with test accuracy of 21.01\% on MNIST, 7.01\% on Fashion-MNIST, and 9.97\% on CIFAR10.

Compared with the Vanilla classifier, all zero knowledge defenses achieve a significant enhancement in terms of test accuracy on all datasets, with the exception of CLP and CLS on CIFAR10 dataset. Among the zero knowledge approaches, ZK-GanDef achieves the highest test accuracy on all the datasets with significant margin. On MNIST, the test accuracy is 88.70\%, 89.29\%, and 98.97\% with CLP, CLS, and ZK-GanDef, respectively. On Fashion-MNIST, the test accuracy is 44.78\%, 41.14\%, and 70.19\% CLP, CLS, and ZK-GanDef, respectively. On CIFAR10, ZK-GanDef is the only zero knowledge defense that reasonably works test accuracy around 60.91\%.

In general, full knowledge approaches have better understanding of the adversarial examples since such examples are part of their training datasets. Therefore, full knowledge approaches should, intuitively, have better test accuracy compared to their zero knowledge counterparts. Our results confirm this observations. The results show that the test accuracy of full knowledge approaches is significantly higher than those of CLP and CLS, especially on Fashion-MNIST and CIFAR10 datasets. On the other hand, the test accuracy of ZK-GanDef is comparable to those of full knowledge defenses. In fact, the test accuracy of ZK-GanDef (98.97\%) is higher than those of all the full knowledge defenses (98.79\%, 97.6\% and 96.85\%). This is because handwritten digits in MNIST are gray scale figures with no detailed texture, and therefore, ZK-GanDef can train to select strongly denoised (even binarized) features without losing information. As a result, ZK-GanDef can achieve even higher test accuracy than full knowledge approaches. 

On Fashion-MNIST, FGSM-Adv achieves the highest test accuracy (90.48\%). The PGD-Adv, PGD-GanDef and ZK-GanDef achieve the second tier test accuracy (76.42\%, 68.19\% and 70.19\%). This is because FGSM-Adv utilizes only original and FGSM examples during training, and therefore, the trained classifier is overfitting on FGSM examples. This behavior has been observed in \cite{tramer2017ensemble} and denoted as gradient masking effect. On CIFAR10, PGD-Adv, PGD-GanDef and ZK-GanDef achieve comparable test accuracy (56.18\%, 54.14\% and 60.19\%, respectively), while the test accuracy of FGSM-Adv is only at 41.53\%. Due to the input dropout in allCNN classifier, the diversity of training data is enhanced and the overfitting of FGSM-Adv is inhibited \cite{tramer2017ensemble}. However, FGSM examples are generated with the weaker  single-step method, and hence the test accuracy degenerates on the stronger iterative examples.\\

%
\begin{figure*}[tb]
\centering
\begin{minipage}[c]{.3\textwidth}
\centering
    \includegraphics[width=\linewidth]{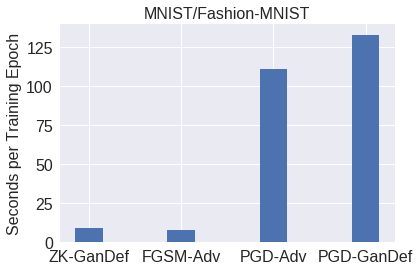}
    \label{fig:time-model-1}
\end{minipage}
\begin{minipage}[c]{.3\textwidth}
\centering
    \includegraphics[width=\linewidth]{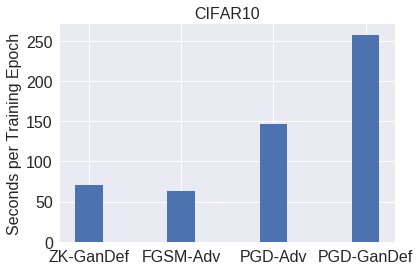}
    \label{fig:time-model-2}
\end{minipage}
\begin{minipage}[c]{.3\textwidth}
\centering
    \includegraphics[width=\linewidth]{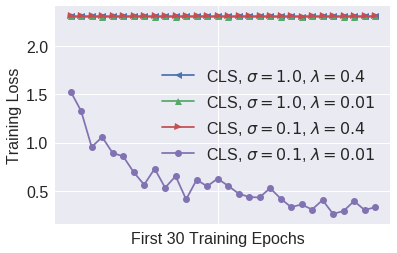}
    \label{fig:loss-cls}
\end{minipage}
\caption{Training Time and Training Loss \textit{The left sub-figure is training time on MNIST and Fashion-MNIST. The middle sub-figure is training time on CIFAR10. The right sub-figure is the training loss of CLS under different hyper-parameters.}}
\label{fig:training-time-loss}
\vspace{-3mm}
\end{figure*}

\subsubsection{On Iterative Adversarial Examples}
We analyze here the test accuracy results on BIM and PGD examples, which are depicted on the sub-figures of the third and the fourth columns of Figure \ref{fig:test-acc}, respectively. The figure clearly shows that the Vanilla classifier completely failed with both BIM and PGD examples. This is because BIM and PGD are iterative adversarial examples and hence are carefully crafted to mislead Vanilla classifiers.


Based on the test accuracy results, using zero knowledge defenses could still enhance the performance on these stronger adversarial examples. However, these enhancements are lower than those on FGSM examples. Among zero knowledge defenses, the test accuracy of ZK-GanDef is significantly higher than those of CLP and CLS on all iterative adversarial examples. On MNIST, the test accuracy of ZK-GanDef with BIM and PGD examples is 25\% and 38\%, respectively, higher than those of CLP and CLS. On Fashion-MNIST, the test accuracy of ZK-GanDef on BIM and PGD examples is 42\% and 47\%, respectively, higher than those of CLP and CLS. On CIFAR10, only ZK-GanDef could work and it achieves 46.27\% and 54.85\% test accuracy on BIM and PGD examples, respectively.

As mentioned earlier, full knowledge defenses could achieve larger enhancement in test accuracy compared to the existing zero knowledge defenses, CLP and CLS. FGSM-Adv is the exception as evidenced by its poor performance in defending iterative adversarial examples due to the reasons we mentioned in the previous sub-section. 
On MNIST and Fashion-MNIST, the test accuracy of FGSM-Adv on BIM and PGD examples has a huge decrease from over 90\% to around 10\%. Although such huge decrease does not exist in the case of CIFAR10, the test accuracy of FGSM-Adv is clearly lower than that of PGD-Adv and PGD-GanDef. On all datasets, PGD-Adv and PGD-GanDef have much stable test accuracy with limited decrease of test accuracy on FGSM examples. More importantly, the results show that the test accuracy of ZK-GanDef is close to those of PGD-Adv and PGD-GanDef on iterative adversarial examples from the three datasets.

We summarize our findings from the results as: (i) ZK-GanDef is significantly better than existing zero knowledge defenses (CLP and CLS) due to its higher test accuracy on adversarial examples and its scalability to large datasets. This clearly supports our vision that utilizing a more flexible and sophisticated way to handle the pre-softmax logits (ZK-GanDef) is better than forcing the pre-softmax logits to be smooth at a small scale (CLP and CLS). (ii) The test accuracy of ZK-GanDef is comparable to that of the state-of-the-art full knowledge adversarial training defenses. This supports our hypothesis that using perturbation invariant features in the classifier could greatly enhance test accuracy on adversarial examples. (iii) On contrast with full knowledge defenses, ZK-GanDef is adaptable to new types of adversarial examples. We see from the results that FGSM-Adv has significant adaptability issue on MNIST and Fashion-MNIST datasets. This issue is not observed on CIFAR10 due to the input dropout in classifier structure \cite{tramer2017ensemble}. For PGD-Adv, the current evaluation does not show its adaptability issue, but it is not guaranteed given that stronger adversarial examples could be generated in the future \cite{tramer2017ensemble}\cite{samangouei2018defense}. On the other hand, the results show that ZK-GanDef has better adaptability to new types of adversarial examples because its training is independent of such examples. 

\subsection{Generalizability}
\textcolor{black}{
In the previous evaluation, all iterative adversarial examples are generated by methods based on projected gradient descent. In order to show the generalizability of ZK-GanDef, we conduct the evaluation on an extra set of adversarial examples, Deepfool \cite{moosavi2016deepfool} and Carlini \& Wagner (CW) examples \cite{carlini2016towards}. Unlike adversarial examples used in previous evaluation, Deepfool and CW adversarial examples contain perturbation patterns that are significantly different from Gaussian perturbation. Therefore, this evaluation could reveal the generalizability of ZK-GanDef in defending other adversarial examples. The evaluation is conducted on all three datasets. The Deepfool and CW adversarial examples utilize the same hyper-parameter setting as PGD adversarial examples.
}

\textcolor{black}{
The evaluation results are summarized in Table \ref{table:eval-deepfool-cw}. It is clear that ZK-GanDef can classify Deepfool adversarial examples with over 85\% accuracy in all three datasets which matches the test error presented in \cite{moosavi2016deepfool}. The reason is that Deepfool tries to find adversarial examples with smaller perturbation than projected gradient descent based adversarial examples (FGSM, BIM, PGD). As a result, Deepfool examples are easier to defend. For CW examples, ZK-GanDef achieves the same level of test accuracy on all three datasets. To conclude, ZK-GanDef is not limited to defend a specific type of perturbation. Although ZK-GanDef only utilizes Gaussian noise perturbation during training, its defense can be generalized to a wide range of adversarial examples which include FGSM, BIM, PGD, Deepfool and CW examples.
}

\textcolor{black}{
\begin{table}[tb]
    \begin{center}
    \begin{tabular}{ c  c | c  c | c  c }
    \hline \hline
    \multicolumn{2}{c |}{MNIST} & \multicolumn{2}{c}{Fashion-MNIST} & \multicolumn{2}{| c}{CIFAR10} \\
    Deepfool & CW & Deepfool & CW & Deepfool & CW \\
    \hline
    98.72\% & 98.46\% & 89.52\% & 66.43\% & 86.08\% & 47.22\% \\
    \hline \hline
    \end{tabular}
    \end{center}
    \caption{Test Accuracy on Deepfool and CW Examples}
    \label{table:eval-deepfool-cw}
\end{table}
}

\subsection{Training Time}

We evaluate here the training time of ZK-GanDef in terms of seconds per training epoch. MNIST and Fashion-MNIST share the same image size and classifier structure and hence has the same training time. Since the test accuracy of ZK-GanDef is significantly higher than those of the existing zero knowledge defenses, CLP and CLS, we only compare the training time of ZK-GanDef with those of full knowledge defenses (FGSM-Adv, PGD-Adv and PGD-GanDef). \textcolor{black}{We utilize a fixed number of training epochs (80 for MNIST and 300 for CIFAR10) and results show that all defensive methods converge at epoch 30 on MNIST and at epoch 240 on CIFAR10. Since the records of training time per epoch have a very small deviation, we take the average value of records in all epochs and compare different defense methods with it.} The results are recorded during the training on a workstation with a NVIDIA GTX 1080 GPU.

The left sub-figure of Figure \ref{fig:training-time-loss} shows that the training time of ZK-GanDef on MNIST/Fashion-MNIST (8.75s) is close to that of FGSM-Adv (7.83s), while it surges to 110.85s and 132.75s in the case of PGD-Adv and PGD-GanDef, respectively. The evaluation results on CIFAR10 dataset (the middle sub-figure of Figure \ref{fig:training-time-loss}) follow a similar trend to that of the results on MNIST and Fashion-MNIST datasets. ZK-GanDef and FGSM-Adv take much less training time per epoch (71.20s and 62.85s, respectively) compared to that of PGD-Adv (146.91s) and that of PGD-GanDef (257.72s). For example, on CIFAR10 dataset, the end-to-end training time of PGD-Adv takes 14.3 hours, while training of ZK-GanDef only takes 6.9 hours. In summary, ZK-GanDef provides test accuracy close to that of the best state-of-the art full knowledge defesnses (PGD-Adv), while reducing the training time by 92.11\% and 51.53\% on MNIST/Fashion-MNIST and CIFAR10, respectively.

\subsection{Convergence Issue}

As presented earlier, the evaluation results of CLP and CLS on CIFAR10 dataset show that these two zero knowledge adversarial training defenses fail to correctly classify both original and adversarial examples. 
This is mainly because the training loss of CLP and CLS does not converge during training. The mathematical models of CLP and CLS (section \ref{sec:defense}) follow the same design logic that aims at preventing over confident predictions. CLP achieves its goal by adding $l_{2}$ norm penalty on the difference of two randomly selected pre-softmax logits, while CLS adds $l_{2}$ norm penalty on any pre-softmax logits. Moreover, CLP and CLS do not include original examples in their training dataset, which means that they miss important features that can help discriminate examples with and without perturbations. Therefore, this design logic is too simple and lacks flexibility compared with ZK-GanDef, which utilizes minimax game with discriminator and trains on examples with and without perturbations. When training on complex datasets like CIFAR10, the simple and less flexible design logic leads to convergence issues for CLP and CLS. 

To further validate this conclusion, we record the loss of CLS during the first 30 training epochs and depict the results on the right sub-figure of Figure \ref{fig:training-time-loss}. The training loss is recorded on four different hyper-parameter settings of CLS: (1) normal CLS ($\sigma = 1.0, \lambda = 0.4$), (2) CLS with reduced perturbations ($\sigma = 1.0, \lambda = 0.01$), (3) CLS with reduced penalty ($\sigma = 0.1, \lambda = 0.4$), and (4) CLS with reduced perturbation and penalty ($\sigma = 0.1, \lambda = 0.01$). The figure shows that the curves of the first three settings overlap with each other and form the  horizontal curve on the top. This clearly shows that CLS does not learn any useful features and hence the training loss does not converge (does not decrease) under these three settings.  Under the last setting, CLS was able to learn useful features and hence the training loss decreases towards convergence. However, with the last setting, CLS falls back to Vanilla classifier, which fails to defend against adversarial examples. A similar experiment is also conducted with CLP and the results follow the same pattern. The only difference is that the training loss goes to ``nan'' on CLP under the first three settings, which means that the classifier diverges during training.

\section{Conclusion} \label{sec:conclusion}

In this paper, we introduce a new zero knowledge adversarial training defense, ZK-GanDef, which combines adversarial training and feature learning to better recognize and identify adversarial examples. We evaluate the test accuracy and the training overhead of ZK-GanDef against state-of-the-art zero knowledge adversarial training defenses (CLP and CLS) as well as full knowledge adversarial training defenses (FGSM-Adv and PGD-Adv). The results show that ZK-GanDef enhances the test accuracy on original and adversarial examples by up to 49.17\% compared to zero knowledge defenses. More importantly, ZK-GanDef has close test accuracy to full knowledge defenses (test accuracy degeneration is below 8.46\%), while taking much less training time (more than 51.53\% on training time reduction). Additionally, in contrast to full knowledge defenses, ZK-GanDaf can adapt to new types of adversarial examples because its training is adversarial example agnostic.

\section{Future Work} \label{sec:future}

In the future, we want to continue our research on designing defensive methods which provide defense against different single-step and iterative adversarial examples while consume less computation during the training.

\balance
\bibliographystyle{IEEEtranS}
\bibliography{reference}

\end{document}